\documentclass[journal,12pt,onecolumn]{IEEEtran}
\usepackage{subfigure, epsfig, color}
\usepackage{amsmath}
\usepackage{amsthm}
\usepackage{mathtools}

\usepackage{algorithm,algorithmic}
\usepackage{amssymb}
\usepackage{gensymb}
\usepackage{datetime}
\usepackage{soul}

\newtheorem{theorem}{Theorem}

\newtheorem{proposition}[theorem]{Proposition}

\begin{document}

\title{Online Antenna Tuning in Heterogeneous Cellular Networks with Deep Reinforcement Learning}

\author{{
    Eren Balevi and
    Jeffrey G. Andrews}\\
\thanks{The authors are with the the Dept. of Electrical and Computer Engineering at the University of Texas at Austin, TX, USA. Email: erenbalevi@utexas.edu, jandrews@ece.utexas.edu. This work has been supported in part by Futurewei (Huawei USA).}				
}
\maketitle \maketitle
\normalsize

\begin{abstract} 
We aim to jointly optimize antenna tilt angle, and vertical and horizontal half-power beamwidths of the macrocells in a heterogeneous cellular network (HetNet). The interactions between the cells, most notably due to their coupled interference render this optimization prohibitively complex. Utilizing a single agent reinforcement learning (RL) algorithm for this optimization becomes quite suboptimum despite its scalability, whereas multi-agent RL algorithms yield better solutions at the expense of scalability. Hence, we propose a compromise algorithm between these two. Specifically, a multi-agent mean field RL algorithm is first utilized in the offline phase so as to transfer information as features for the second (online) phase single agent RL algorithm, which employs a deep neural network to learn users locations. This two-step approach is a practical solution for real deployments, which should automatically adapt to environmental changes in the network. Our results illustrate that the proposed algorithm approaches the performance of the multi-agent RL, which requires millions of trials, with hundreds of online trials, assuming relatively low environmental dynamics, and performs much better than a single agent RL. Furthermore, the proposed algorithm is compact and implementable, and empirically appears to provide a performance guarantee regardless of the amount of environmental dynamics.
\end{abstract}

\begin{IEEEkeywords}
Deep reinforcement learning, online antenna tuning, Q-learning, HetNets, 5G. 
\end{IEEEkeywords}

\section{Introduction}
Cellular networks will rapidly densify for 5G and beyond, largely through the opportunistic addition of small cells over time \cite{Andrews5G}.  A significant performance advantage of cellular networks (as opposed to e.g. WiFi) is the use of directional antennas at the base station, which concentrates transmit (and receive) power both horizontally and vertically, as well as having an appropriate downtilt.  Properly setting these three antenna parameters has important ramifications not only on the desired received power by users (UEs) but also on the interference due to neighboring cells.  Unfortunately, a centralized optimum solution of these is NP-hard, as well as impractical since the UE locations are unknown\footnote{UEs do not explicitly return their locations in radio measurements in LTE.}. Traditionally these settings have been implemented manually and/or by trial and error, but clearly this is far from optimal as well as not scalable for dense networks, particularly when new base stations can be added at any time or other aspects of the environment can change.  

The goal of this paper is to develop a scalable and distributed near-optimal method for setting these antenna parameters.  In particular, we wish to dynamically maximize the (arbitrarily weighted) sum-rate of the users in the network by having each base station autonomously set their beamwidths and tilt based on ongoing feedback from the users in their own cell.  Since the optimum antenna settings change relatively slowly -- the cellular topology and propagation conditions change slowly -- deep reinforcement learning is a promising tool which has had significant recent success in challenging problems with analogous characteristics \cite{Mnih}-\cite{Lillicrap}.  

\subsection{Related Work and Motivation}
Tuning the antenna parameters of macrocells has been extensively studied in the literature under the mantle of self organizing networks (SONs) \cite{survey13}. Many works in this line of research have focused on the antenna tilt angle utilizing methods from conventional optimization theory \cite{Boyd} in an attempt to optimize the capacity and coverage  \cite{Buenestado}-\cite{Awada}. These optimizations tend to be restricted to some special handcrafted rules and heuristics, and do not learn the dynamics of the environment. This leads to the loss of adaptability especially for a time-varying environment. 

Reinforcement learning (RL) on the other hand can learn and adapt to the dynamics of the environment. One of the first attempts to make use of RL for the antenna tilt optimization was for single-tier cellular networks \cite{Razavi}. The major limitation of this study is that it requires a single macrocell at a time to optimize its parameters in a static environment. This assumption was relaxed in \cite{Islam} by allowing macrocells to simultaneously optimize their parameters. Both \cite{Razavi} and \cite{Islam} addressed \textit{the-curse-of-dimensionality} problem in RL via a combination of fuzzy logic and RL \cite{Bonarini}. This fuzzy RL framework was also used for the optimization of the antenna tilt and transmission power \cite{Cigdem}. However, fuzzy RL activates more than one membership function while discretizing the continuous state and action variables,  which complicates parsing the reward signal. Sparse sampling is another technique that can handle the-curse-of-dimensionality problem \cite{Kearns}. This method was utilized for the self-optimization of the coverage through the antenna tilt optimization \cite{Thampi}. The recent advances in deep RL reveal that neither fuzzy RL nor sparse sampling is as efficient as deep learning (DL) based models.

One paper \cite{Guo} that considers a HetNet, as opposed to a single-tier of base stations, attempts to increase the user fairness via a dynamic antenna tilt optimization. However, their model does not address the scalability problem, and hence is limited to a simple environment with very few cells.  Motivated by the recent success of deep RL algorithms \cite{Mnih}-\cite{Lillicrap}, our paper aims to develop an online deep RL-based antenna tuning algorithm for a more complex and realistic network topology, e.g. for highly dynamic HetNets and with all three key antenna parameters considered simultaneously.

\subsection{Contributions} \label{Contributions}
The main contribution of this paper is a novel and practical deep RL algorithm for optimizing the antenna parameters of macrocells in HetNets in an attempt to maximize the weighted sum-rate of users. For this purpose, we first design a RL-based framework to transform the weighted sum-rate optimization problem into a Markov Decision Process (MDP) by defining states, actions and rewards according to the available control signals in macrocells. The proposed framework is used for jointly tuning the antenna parameters of macrocells, which are the antenna tilt angle, and vertical and horizontal half-power beamwidths, so as to maximize the weighted sum-rate. However, solving this MDP for multi-cell environments, in which each macrocell acts as an agent, yields an exponentially growing action space \cite{NashQ}. On the other hand, considering all neighboring cells as a part of the environment to formulate the problem as a single agent RL becomes highly suboptimum despite its scalability \cite{Tan}. 

We propose a two-step novel practical deep RL algorithm as a compromise solution between single agent and multi-agent RL for optimizing the antenna parameters of macrocells in HetNets. In the first offline step, a multi-agent mean field RL algorithm, which treats all neighboring agents as a single virtual agent  \cite{ML-MARL}, is leveraged to make the RL problem scalable and tractable. Specifically, mean field RL enables us to consider the inter-cell interference -- which is the main impediment in weighted sum-rate optimization problems -- as the cumulative policy of neighboring cells. In the second online step, the learned behaviors of neighboring cells (or the average aggregate inter-cell interference due to having different antenna settings for a time period) and the learned locations of UEs with a deep neural network (DNN) by exploiting the correlations among signal-to-interference-plus-noise ratios (SINRs) are utilized as features for the single agent feature based Q-learning that uses linear function approximation. This selection of features decreases the required online trails from millions to hundreds, because it enables us to estimate the immediate impact of any possible antenna setting without actually taking the time to test it, which would require setting it and then processing feedback from the users. The other key insights for the proposed deep RL algorithm are as follows:
\begin{itemize}
\item{Despite having just hundreds online trials and with only having local observations, the proposed algorithm can approach the performance of the optimum solution, which can be achieved by solving the multi-agent mean field Q-learning algorithm online that requires millions of trials with a global network knowledge, for relatively low environmental dynamics. Additionally, our two-step algorithm is much better than a classical single agent RL algorithm, which is totally nonadaptive to the neighboring cells.} 
\item{Even if the relative variance of the environment dynamics is high, the proposed algorithm appears to guarantee a certain performance gain. Specifically in our simulations, regardless of how high the relative variance is, we empirically observe about $0.6\Delta$ dB SINR gain when the optimum solution provides $\Delta$ dB SINR gain.}
\item{The proposed algorithm is quite compact in that $80\%$ of the states in the state space are mapped to a small subset of the action space. Because many states map to the same actions (antenna settings), the antenna arrays will not have to be changed very often.}
\end{itemize}

The paper is organized as follows. The system model and problem statement are given in Section \ref{Problem Statement}. The proposed RL framework for weighted sum-rate maximization is introduced in Section \ref{Problem Formulation}, which is employed in developing the deep RL algorithm for dynamically tuning the antenna parameters of macrocells in HetNets in Section \ref{DeepNodeB Model}. The simulation results are given in Section \ref{Simulations} and the paper ends with concluding remarks and suggested future work in Section \ref{Conclusions}.  

\section{System Model and Problem Statement}\label{Problem Statement}
We consider an urban HetNet deployment, which is composed of many macrocells and small cells. Macrocells are one of the three sectors of an eNodeB (eNB, i.e. a base station). Small cells refer to pico or femtocells \cite{SevenWays}. All the parameters of cells are initially set to their default values, e.g., see \cite{Macro-Pico} for the 3GPP Release-13. According to these settings, the uniformly distributed UEs are connected to one of the cells. Specifically, each UE is assigned to the cell from which it receives the strongest signal. Under this scenario, i.e., after cell selection each macrocell dynamically tunes its antenna-related parameters uniquely according to their UEs and environment. In particular, the antenna tilt angle $\theta_t$, and vertical and horizontal half-power beamwidths, denoted by  $\theta_{3dB}$ and $\phi_{3dB}$, of macrocells are jointly tuned, e.g., $\psi_k=(\theta_{t,k}, \theta_{3dB,k}, \phi_{3dB,k})$ for the $k^{th}$ macrocell. Optimizing these parameters considering the environment (or interference) can greatly enhance SINRs. To illustrate, the antenna gain for a particular horizontal and vertical angle is \cite{Rel10} 
\begin{equation} \label{antenna_gain}
A(\phi, \theta) = -\min \left[ - (A^{(h)}(\phi) + A^{(v)}(\theta)), 25 \right] \text{dB}
\end{equation}
where
\begin{equation} 
A^{(h)}(\phi) = -\min \left[12\left(\frac{\phi}{\phi_{3dB}}\right)^2, 25 \right] \text{dB}
\end{equation}
and
\begin{equation} 
A^{(v)}(\theta) = -\min \left[12\left(\frac{\theta-\theta_{t}}{\theta_{3dB}}\right)^2, 20 \right] \text{dB}.
\end{equation}
In words, an inappropriate antenna setting can lead to a $20$ or $25$ dB power loss for a UE. 

There can be a large gain by dynamically configuring the antenna parameters of each macrocell uniquely according to its UEs and environment. However, it is not easy to find the optimum antenna configuration for a macrocell because of the conflicts among UEs and coupling amongst cell nearby base station settings. For example, one antenna configuration at one base station can increase the data rate of one UE while simultaneously decreasing the data rate of another UE in its own or neighboring cell. Hence, configuring the antenna parameters is formulated as an optimization problem so as to maximize the weighted sum-rate of users while guaranteeing $1-\epsilon_{cov}$ coverage. In this manner, each macrocell optimizes the antenna parameters according to the typical low mobility UEs that use that macrocell most of the time. Note that it is not reasonable to tune the antenna parameters based on the highly mobile UEs, which anyway will only be present in the cell for a short time.

A key point when it comes to optimizing antenna parameters is that the acquired solution has to be utilized for a fairly long period of time, because it is not practical to change the antenna settings very frequently. Thus, the optimization problem is expressed for $K$ macrocells, each of which has $U$ number of typical UEs with $V$ spatial streams as

\begin{equation} \label{optimizationProblemFormNW}
\begin{aligned}
& \underset{ \psi=(\psi_1, \cdots, \psi_K)}{\text{maximize}}
& & \mathbb{E}\left[\sum_{n=1}^N\sum_{k=1}^K\sum_{u=1}^U\sum_{v=1}^{V}\lambda_{k,u,v}[n]\log_2(1+\rho_{k,u,v}[n])\right]  \\ 
& \text{subject to} & &  \mathbb{P}[\rho_{k,u,v}[n]<\rho_{min}] <\epsilon_{cov}, \forall n,k,u,v \\
& & & \theta_{t,k}^{\min}\leq \theta_{t,k} \leq \theta_{t,k}^{\max}, \forall k \\
& & & \theta_{3dB,k}^{\min}\leq \theta_{3dB,k} \leq \theta_{3dB,k}^{\max}, \forall k\\
& & & \phi_{3dB,k}^{\min}\leq \phi_{3dB,k} \leq \phi_{3dB,k}^{\max}, \forall k\\
\end{aligned}
\end{equation}
The first constraint is to ensure $1-\epsilon_{cov}$ coverage in which $\rho_{min}$ is the minimum necessary SINR for a successful transmission, and all the other constraints satisfy the minimum and maximum allowed values of the optimization variables.  The $\lambda_{k,u,v}[n]$'s are the nonnegative input parameters and can be determined to prioritize users or control data rate increase in cell-interior or cell-edge. Note that the SINR is
\begin{equation} \label{SINR_exp}
\rho_{k,u,v}[n]=\frac{P_{k,u,v}[n]A_{k,u,v}[n]G_{k,u,v}[n]}{\sum_{j \neq k}P_{j,u,v}[n]A_{j,u,v}[n]G_{j,u,v}[n]+\sigma_n^2}
\end{equation}
where $P_{k,u,v}[n]$, $A_{k,u,v}[n]$\footnote{The $\phi, \theta$ terms in antenna gain are omitted for brevity.} and $G_{k,u,v}[n]$ are the transmit power, antenna gain and channel gain respectively. The noise power is denoted as $\sigma_n^2$. Throughout the paper, it is assumed that macrocells transmit at full power with rate adaptation being used to track channel variations, as is typical in downlink cellular systems.

The primary challenges regarding the optimization problem \eqref{optimizationProblemFormNW} in addition to it being non-causal due to requiring a decision at current time for the next $N$ time intervals are as follows. This is a non-convex optimization problem and the network-wide global optimum solution can be found by a centralized station with many control signaling only through an exhaustive search or the branch-and-bound algorithm \cite{Book-RA}. Clearly, this is not a scalable approach. More importantly, this optimization is NP-hard due to inter-cell interference \cite{Luo-Zhang}. Given these challenges, this paper proposes to approach the solution of  (\ref{optimizationProblemFormNW}) in a distributed way, which is given by

\begin{equation} \label{optimizationProblemForm}
\begin{aligned}
& \underset{\psi=(\theta_{t}, \theta_{3dB}, \phi_{3dB})}{\text{maximize}}
& &  \mathbb{E}\left[\sum_{n=1}^N\sum_{u=1}^U\sum_{v=1}^{V}\lambda_{u,v}[n]\log_2(1+\rho_{u,v}[n])\right]  \\ 
& \text{subject to} & &  \mathbb{P}[\rho_{u,v}[n]<\rho_{min}] <\epsilon_{cov}, \forall n,u,v \\
& & & \theta_{t}^{\min}\leq \theta_{t} \leq \theta_{t}^{\max} \\
& & & \theta_{3dB}^{\min}\leq \theta_{3dB} \leq \theta_{3dB}^{\max}\\
& & & \phi_{3dB}^{\min}\leq \phi_{3dB} \leq \phi_{3dB}^{\max}\\
\end{aligned}
\end{equation}
by having each macrocell individually using an RL-based algorithm. This also addresses the causality problem. Specifically, a novel practical deep RL algorithm is developed where each macrocell makes autonomous decisions. Although the performance gap between this distributed solution and the network-wide globally optimum solution is very difficult to quantify exactly, it is valuable to find distributed solutions for this complex weighted sum-rate optimization problem as stated in \cite{Book-RA} and references therein. 
 
\section{A Reinforcement Learning Framework for Weighted Sum-Rate Maximization}\label{Problem Formulation}
\begin{figure*}[!t]
\centering
\subfigure[]{
\label{fig:RLmodel}
\includegraphics[width=4.25in]{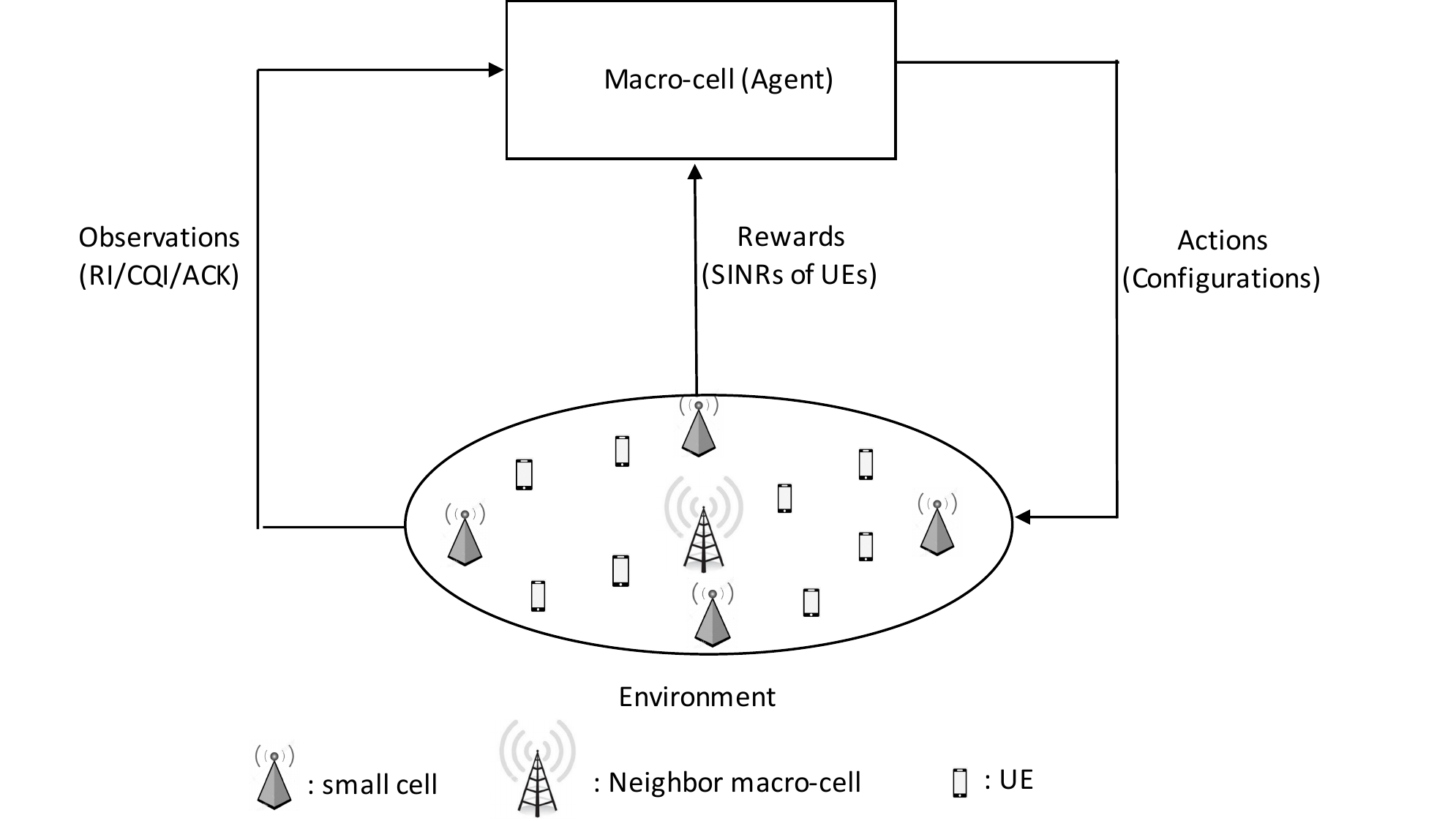}}
\qquad
\subfigure[]{
\label{fig:seqOrder}
\includegraphics[width=2.25in]{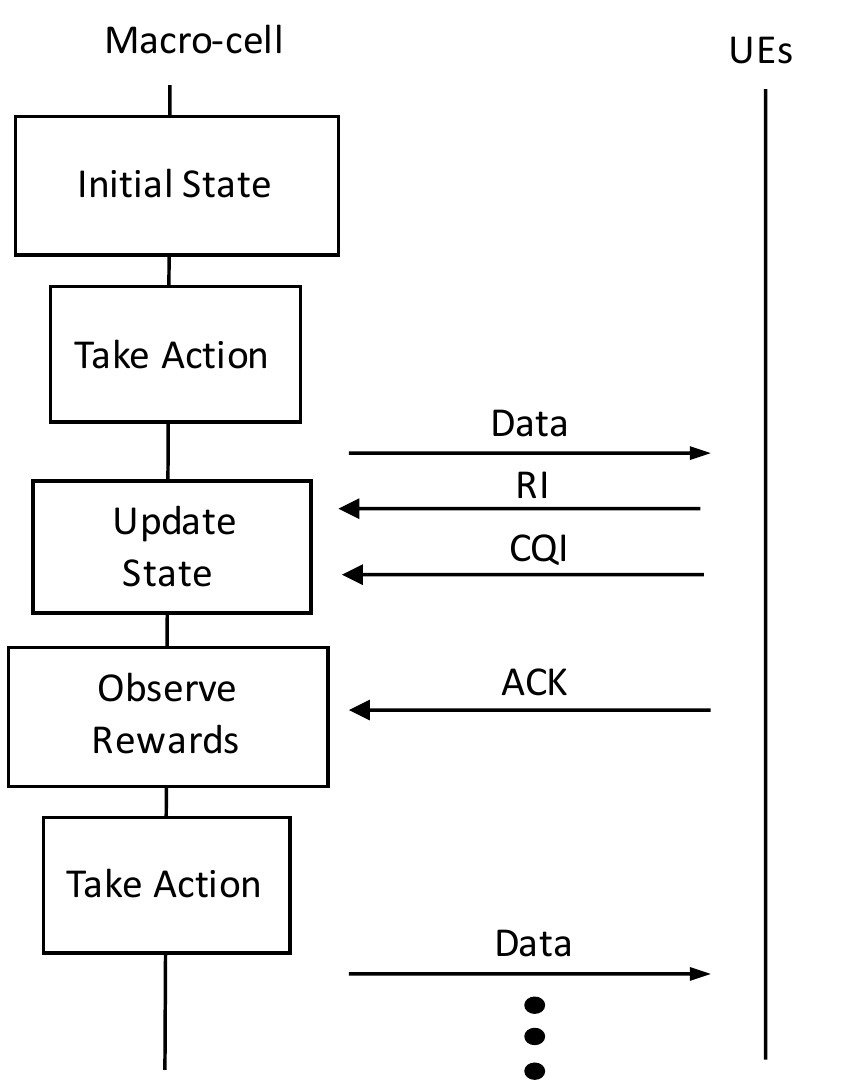}}
\caption{The proposed reinforcement learning framework: (a) actions refer to the possible antenna configurations, and rewards are taken in case of successful transmissions, (b) the sequential order of updating states, taking actions and observing rewards.}
\end{figure*}
In conventional RL problems, agents learn the environment through observations before choosing an action, and rewards are received in response to actions \cite{Barto}. To adapt this RL framework to our distributed optimization problem, macrocells are treated as agents, and the HetNet (from a single macrocell point of view) is the environment as depicted in Fig. \ref{fig:RLmodel}. In our problem formulation, we utilize standard UE feedback such as rank indicator (RI), channel quality indicator (CQI), and acknowledgments (ACKs), where RI indicates the number of transmission layers that UE can distinguish, CQI discloses the quality of the channel, and ACKs/NACKs notify which transmission succeeded (if SINR is greater than $\rho_{min}$, UE sends the ACK signal to the macrocell) or failed. 

The aforementioned RL framework is employed to solve the weighted sum-rate optimization problem by first converting it to an MDP. However, this is nontrivial and its efficiency depends on the selection of states, actions and rewards. For this purpose, we uniquely define states, actions and craft a reward mechanism. An action is taken for each state that leads to receiving a reward, whose cycle is depicted in Fig. \ref{fig:seqOrder}. This cycle is repeated throughout the training period. Once the optimum action is found according to the expected long term rewards, it remains fixed for a long period of time and would only change if significant sustained changes in the environment were observed, e.g., if a new base station was added nearby. 

\subsection{State Space}
The environment can be represented as a state space such that the current state gives the latest status of the environment. Since interference is seen as part of the environment, states are defined as the SINRs of UEs. More specifically, states are defined in terms of quantized SINRs (in dB scale) to reduce the dimensionality of the problem. Macrocells obtain quantized SINRs from CQIs. Specifically, the $v^{th}$ spatial stream of the $u^{th}$ UE has a quantized SINR $\gamma_{u,v}^{(q)}[n]$ for a macrocell such that $\gamma_{u,v}[n] = 10\log_{10}(\rho_{u,v}[n])$\footnote{The quantized and unquantized SINRs are denoted as $\gamma_{u,v}^{(q)}[n],  \gamma_{u,v}[n]$, respectively in dB scale.}. Since CQI observations are at transmission time interval (TTI) level and the antenna setting may not practically be done at that rate, multiple CQIs are received for each UE in one period, in which the period is defined as the inverse of the antenna setting rate. Hence, the state is defined by averaging the quantized SINR values\footnote{Notice that fast fading disappears due to this averaging.}. The concatenation of all quantized and averaged SINRs defines the state of the environment in macrocells. In accordance with this definition, the state at the $n^{th}$ period becomes
\begin{equation} \label{stateDef}
s=[\gamma_{1,1}^{(q)}[n] \cdots \gamma_{U,V}^{(q)}[n]]
\end{equation}
where $s\in \mathbb{S}$, and $\mathbb{S}$ is the set of all possible combinations. For a quantized SINR with $M$ bins there are  
\begin{equation} \label{stateNumber}
|\mathbb{S}|=M^{N_{str}}
\end{equation}
possible environment states, in which $N_{str}$ is the total number of streams.

\subsection{Action Space}
Each action corresponds to one possible antenna configuration of a macrocell in terms of the antenna tilt angle $\theta_t$, and vertical and horizontal half-power beamwidths denoted as $\theta_{3dB}$ and $\phi_{3dB}$. Hence, one action $a$ is described as 
\begin{equation} \label{actionDef}
a=(\theta_t, \theta_{3dB}, \phi_{3dB})
\end{equation}
where $a \in \mathbb{A}$. The cardinality of $\mathbb{A}$ or the possible number of actions can be expressed as
\begin{equation} \label{actionNumber}
|\mathbb{A}| = |\theta_{t}||\theta_{3dB}| |\phi_{3dB}|.
\end{equation}
All set of actions constitute the action space $\mathbb{A}$, which can be treated as having a probability distribution function over actions, i.e., associating each action to a probability, which is also called a policy. Initially, all actions have some probabilities over a state and these probabilities are updated with rewards in the training period. 

\subsection{Rewards}
A reward mechanism has to be crafted in addition to defining problem-specific states and actions. For this purpose, the immediate reward taken one step ahead of the selected action is designed. An immediate reward is received when ACKs come from the typical UEs; otherwise a large penalty is applied. This is directly related with the objective function and the first constraint of (\ref{optimizationProblemForm}). More precisely, suppose that the quantized SINR of the $v^{th}$ stream of the $u^{th}$ UE is $\gamma_{u,v}^{(q)}[n]$ for a macrocell. Then, the immediate reward becomes
\begin{equation} \label{rewardImm}
r[n] = \sum_{u=1}^U \sum_{v=1}^Vc_{u,v}[n]f_r(\gamma_{u,v}^{(q)}[n])
\end{equation}
where
\begin{align} \label{rewardGen}
    f_r(\gamma_{u,v}^{(q)}[n]) = 
    \begin{cases}
       10(\log_{10}2)\log_2(1+10^{\frac{\gamma_{u,v}^{(q)}[n]}{10}}) &  \text{if ACK}\\
       \xi & \text{o.w.}
    \end{cases} 
\end{align}
If $\gamma_{u,v}[n]$, which is determined by the chosen action and environment, is greater than $\gamma_{min}$, which is equal to $10\log_{10}(\rho_{min})$, UE sends the ACK signal to the macrocell; otherwise a large penalty is received, which is represented by $\xi$. Then, the expected long-term rewards can be trivially stated in terms of the immediate rewards in (\ref{rewardImm}) for a given state and action as 
\begin{equation}  \label{rewardr}
R=\mathbb{E}[ r[n]+\alpha r[n+1] + \alpha^2 r[n+2] + \cdots | \mathbb{S}=s, \mathbb{A}=a]
\end{equation}
where $\alpha$ is the discount factor that determines the importance of future rewards, which can be found via Q-learning as detailed next.

\subsection{Problem Formulation}
The defined states and actions in (\ref{stateDef}), (\ref{actionDef}) as well as the crafted reward mechanism in (\ref{rewardImm}) leads to the canonical RL problem of\footnote{Since a single-agent RL algorithm is used in the online phase, the problem is formulated according to single agent RL.} 
\begin{equation} \label{optimizationProblemRew}
\begin{aligned}
& \underset{ a }{\text{maximize}}
& & R \\
& \text{subject to} 
& & & a\in \mathbb{A}. \\
\end{aligned}
\end{equation}
The optimization problem of (\ref{optimizationProblemForm})  and (\ref{optimizationProblemRew}) give the same solution for quantized SINRs when $c_{u,v}[n]=\lambda_{u,v}[n]/(\alpha^n10\log_{10}2)$, which can be readily shown as follows. The objective function and the first constraint of (\ref{optimizationProblemForm}) are ensured by (\ref{rewardImm}) in (\ref{optimizationProblemRew}). Since $a$ corresponds to $\psi$, the other constraints of (\ref{optimizationProblemForm}) are held by the first constraint of (\ref{optimizationProblemRew}). 

Solving an RL problem refers to finding the optimum action for a given state through rewards, which can be seen as finding the optimum policy. If the state transitions and rewards were given as $\mathit{a\ priori}$ information, the optimum policy would be readily found with dynamic programming by the agents. However, it is not possible to know these dynamics in our problem. Thus, we use online temporal-difference (TD) learning in which neither state transitions nor rewards are known as a priori information, and the agent acts in the environment and uses its experience to find the optimum policy. Specifically, Q-learning, which is one of the most appropriate methods for TD learning, is employed to solve the problem in (\ref{optimizationProblemRew}). 

\section{A Deep Reinforcement Learning Algorithm For Antenna Tuning} \label{DeepNodeB Model}
In a distributed optimization, the decision given by a macrocell, which optimizes the antenna parameters from its local observations, triggers the future interference values of its UEs because of the interactions among neighboring cells. To capture these interactions, a multi-agent RL modeling is necessary. However, multi-agent RL leads to exponentially increasing dimensions, and infeasible or intractable learning due to the increased number of agents as well as instability \cite{Schutter}. This limits the relevance of multi-agent RL for HetNets, which would have a very large number of agents. On the other hand, considering all the neighboring macrocells as a part of the environment to use a single-agent RL algorithm leads to a highly suboptimum solution \cite{Tan}. 

We propose a novel deep RL algorithm that is a compromise between multi-agent and single agent RL based on the recently proposed mean field RL algorithm that makes multi-agent RL scalable \cite{ML-MARL}. In accordance with that, our algorithm is composed of two-steps. In the first step, the multi-agent mean field RL algorithm is used to learn the cumulative behavior of the neighboring macrocells for each state-action pair of the target macrocell. This step is offline, i.e., a realistic simulation environment is constituted for the macrocells whose own locations as well as the locations of their typical UEs are known. Since the locations of small cells and their behaviors in response to the actions of macrocells are unknown, the first step itself is not sufficient to learn the  stochastic environment. Furthermore, there can be some mismatches between the simulation and real environment. These factors explain why we need the online step. In the second or online step, a single-agent RL or in particular a feature-based Q-learning is utilized by defining the learned cumulative behavior of the neighboring macrocells in the first step as features. The underlying reason of having a single agent RL algorithm for the online phase is associated with complexity. That is, if the multi-agent mean field RL was used in the online phase, this would require millions of trials and the global network knowledge. On the other hand, hundreds of online trials are sufficient for our proposed algorithm and the global network knowledge is not needed, i.e., each macrocell takes online actions according to its local observations at the expense of some performance loss. These offline and online steps are elaborated next. In what follows, the pseudo-code of the proposed algorithm is presented, and its convergence and complexity analyses are given.

\subsection{Offline Multi-Agent Mean Field Reinforcement Learning}
Solving a multi-agent RL problem with tens or hundreds of agents is not scalable, in particular it is impractical even in the simulation environment. Mean field RL was proposed as a remedy to cope with the curse of dimensionality in multi-agent RL problems by approximating the average effect of all neighboring agents as a single virtual agent \cite{ML-MARL}. This reduces the RL problem to two agents: the target agent and the virtual agent. Here, each agent's behaviors depend on the dynamics of the population, and the dynamics of the population change by the actions of each agent. The interplay between these two entities reinforces each other and comes to a Nash equilibrium. Interestingly, this mean field RL approach is inherently applicable to the optimization problems that aim to maximize SINRs, because it is sufficient to adapt to the cumulative effect (i.e. the cumulative interference) of the neighboring cells to make the optimum decision for the target cell instead of adapting their individual responses.  

The interference term in (\ref{SINR_exp}) is a random variable
\begin{equation} \label{interference}
I_{j,u,v}[n] = \sum_{j \neq k}P_{j,u,v}[n]A_{j,u,v}[n]G_{j,u,v}[n],
\end{equation}
whose distribution depends on the parameters of the macrocells and small cells, and the channel, which can also be expressed as 
\begin{equation} \label{SCinterference}
I_{j,u,v}[n] = \sum_{j\neq k, j\in M_c}P_{j,u,v}[n]A_{j,u,v}[n]G_{j,u,v}[n] + \sum_{j\in S_c}P_{j,u,v}[n]A_{j,u,v}[n]G_{j,u,v}[n],
\end{equation}
where $M_c$ is the set of macrocells and $S_c$ is the set of small cells. Associating $A_{j,u,v}[n]$ with its corresponding action $a_j$, the mean action of the $N_j$ number of neighboring macrocells to the $k^{th}$ macrocell becomes
\begin{equation}\label{MCinterference}
\bar{a}_k = \frac{1}{N_j}\sum_{j=1}^{N_j}a_j
\end{equation}
and 
\begin{equation}\label{targetMLaction}
a= {\rm argmax\ } \pi_k(a_k | s,\bar{a}_k ) 
\end{equation}
for the $k^{th}$ agent policy $\pi_k$. Iterating \eqref{MCinterference} and \eqref{targetMLaction} via the mean field RL algorithm in \cite{ML-MARL} converges to a solution. This enables us to define
\begin{equation} \label{betadef}
\beta(s,a) \coloneqq \mathbb{E}\left[\sum_{j\neq k, j\in M_c}P_{j,u,v}[n]A_{j,u,v}[n]G_{j,u,v}[n]\right],
\end{equation}
which is used while defining features in the online phase as an approximation to the interference. The mismatch between the true and approximated interference is quantified in evaluating the performance of the algorithm by defining a metric, which is the relative variance of the environment dynamics, given by
\begin{equation} \label{env_dyn}
\eta = \frac{{\rm var}(\sum_{j\in S_c} P_{j,u,v}[n]A_{j,u,v}[n]G_{j,u,v}[n])}{\mathbb{E}_{s,a}[\beta(s,a)]}.
\end{equation}

\subsection{Online Single Agent Feature-Based Q-learning}
A single agent feature-based Q-learning algorithm is devised in this phase, in which the features are obtained with the help of the mean field RL and DL. To be self-contained, the basics of Q-learning is briefly summarized in the Appendix. A key point for Q-learning in wireless networks is the value of $\epsilon$ that is used to pick an action for the $\epsilon$-greedy policy. In particular, RL problems in wireless networks are much more sensitive to the value of $\epsilon$ than in other fields due to wasting scarce resources such as power and bandwidth, which are consumed in proportion to the number of trials. Due to that, we empirically develop a scheduling algorithm for $\epsilon$-greedy policy, given by 
\begin{equation}  \label{epsilon-scheduling}  
\epsilon = 1/k
\end{equation}
where $k$ is initially taken as $1$ and increases with period $T_\epsilon$, that is $k \rightarrow k+1, \text{if}\ n\ \text{mod}\ T_\epsilon = 0$.

Our feature-based Q-learning algorithm relies on linear function approximation. This is because a linear function approximation for Q-learning has a convergence guarantee, whereas nonlinear function approximators such as DNNs do not have any convergence guarantee \cite{Barto}. In linear function approximation, the q-values in (\ref{qFunc}) are approximated as the linear combination of features given by
\begin{equation} \label{refVec}
\hat{q}(s,a,\textbf{w}) = \textbf{x}^T\textbf{w}
\end{equation}
where $\textbf{x}$ is the feature vector, and $\textbf{w}$ can be found by defining a cost function that minimizes the mean square error between the actual and approximated q-values as
\begin{equation} \label{costFnc}
J(\textbf{w}) = \mathbb{E}[ ( q_{\pi_n}(s,a) - \hat{q}(s,a,\textbf{w}) )^2 ]
\end{equation}
where
\begin{equation} \label{TD}
q_{\pi_n}(s,a) = r[n] + \alpha \hat{q}(s',a',\textbf{w})
\end{equation}
due to TD learning. Since (\ref{costFnc}) is a quadratic function with respect to $\textbf{w}$, and stochastic gradient descent removes the expected value, $\textbf{w}$ is iteratively updated as
\begin{equation} \label{updatew2}
\textbf{w}\leftarrow \textbf{w} + \mu ( r[n] + \alpha \hat{q}(s',a',\textbf{w}) - \textbf{x}^T\textbf{w} ) \textbf{x}.
\end{equation}

The key point in feature-based Q-learning is to craft features, i.e., design $\textbf{x}= [x_{1,1} \cdots x_{U,V}]$. Although it is well-known that each feature should give some information about the impact of the selected action, the overall feature selection process is a nontrivial task. We define the features to be able to estimate the immediate rewards in (\ref{rewardImm}) without actually taking the action. The main benefit of this is to decrease the number of online trails. Accordingly, the features are defined in terms of SINRs, that is,
\begin{equation} \label{refVec2}
x_{u,v} = \gamma_{u,v}^{(q)}[0] + \Delta A_{u,v} - \Delta\beta(s,a),
\end{equation}
where the first term on the right hand side of (\ref{refVec2}) denotes the initial quantized SINR, the second term is the difference between the antenna gain due to the selected antenna setting at the current time $n$ and the initial setting, which is
\begin{equation} \label{DeltaAntenna}
\Delta A_{u,v} = A_{n}(\phi_{u,v}, \theta_{u,v}) - A_{0}(\phi_{u,v}, \theta_{u,v}),
\end{equation}
and the last term is the change in the interference due to taking the current action for the current state, given by
\begin{equation} 
\Delta\beta(s,a) = \beta(s,a)-\beta_{0}
\end{equation}
where $ \beta(s, a)$ is defined in \eqref{betadef} and enables us to become adaptive to the neighboring macrocells even if a single agent RL algorithm is used, and $\beta_{0}$ is the initial interference.

It is apparent that this feature selection requires us to know the locations of UEs due to $\phi_{u,v}$ and $\theta_{u,v}$ in \eqref{DeltaAntenna}, which are not known by macrocells. To leverage these features and propose an algorithm accordingly, a supervised DNN is designed so as to learn the locations of UEs at macrocells by exploiting the correlations among their SINRs. For this problem we do not have to learn the UE locations very accurately, because rewards are based on quantized SINRs. Hence, the locations are learned in a cluster basis, meaning that the coverage area of a macrocell is divided into clusters, and then which UEs are in which clusters is learned. Our approach is different than the prior works that find the UEs locations via measuring the multipath characteristics of the received signal \cite{Wax}, \cite{Kupershtein} or the received signal strength \cite{Bahl}, \cite{Laitinen} as the location fingerprint. Specifically, we exploit a pattern among clusters in terms of SINR to find the UE locations as opposed to using a database.

Clusters are formed in polar coordinates directly related to the aim of optimizing the antenna parameters. These clusters for a macrocell are illustrated in Fig. \ref{fig:PolarAE} when there are $N$ clusters, e.g., it is $20$ for this illustration. Every cluster has a value $i_k$ for $k=1, 2, \cdots, N$, which shows the average SINRs of UEs in that cluster assuming that there can be more than one UEs in one cluster. It is highly unlikely that any two clusters have the same average value because of the HetNet (asymmetric) topology. The cluster values vary according to a pattern associated with the environment. The goal is to learn these cluster values by exploiting this pattern using a supervised DNN. Once cluster values are learned, UEs can be mapped to the clusters, in particular to the one with the closest SINR.  

\begin{figure}[!t]
\centering
\subfigure[]{
\label{fig:PolarAE}
\includegraphics[width=3.60in]{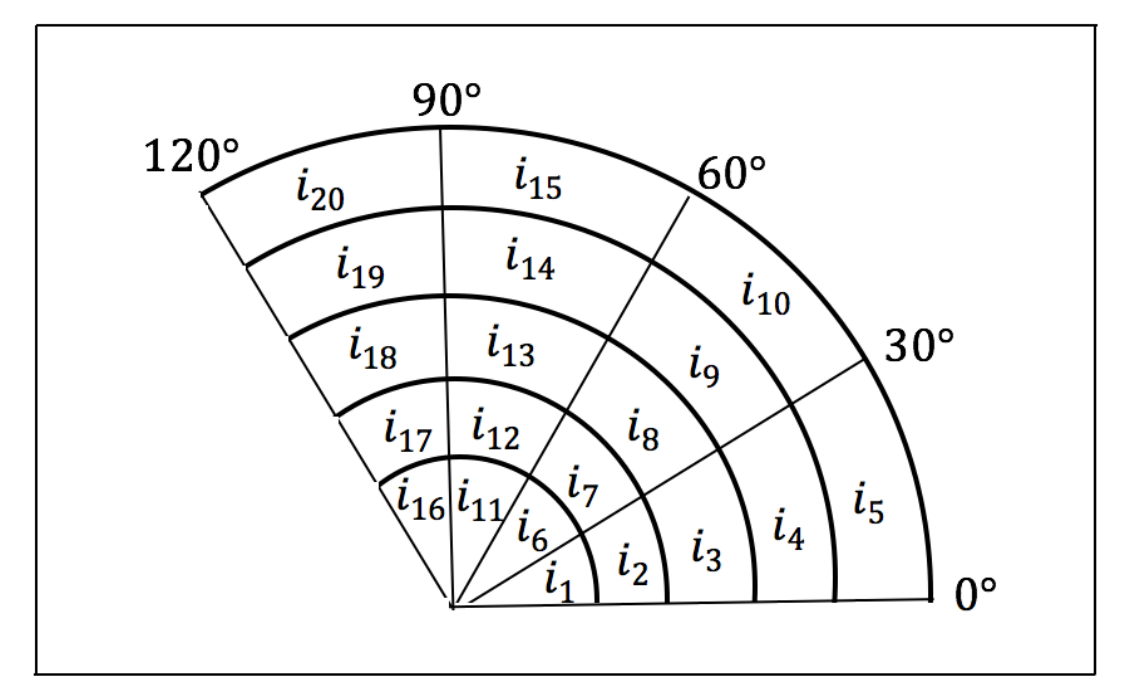}}
\qquad
\subfigure[]{
\label{fig:AE_arch}
\includegraphics[width=2.95in]{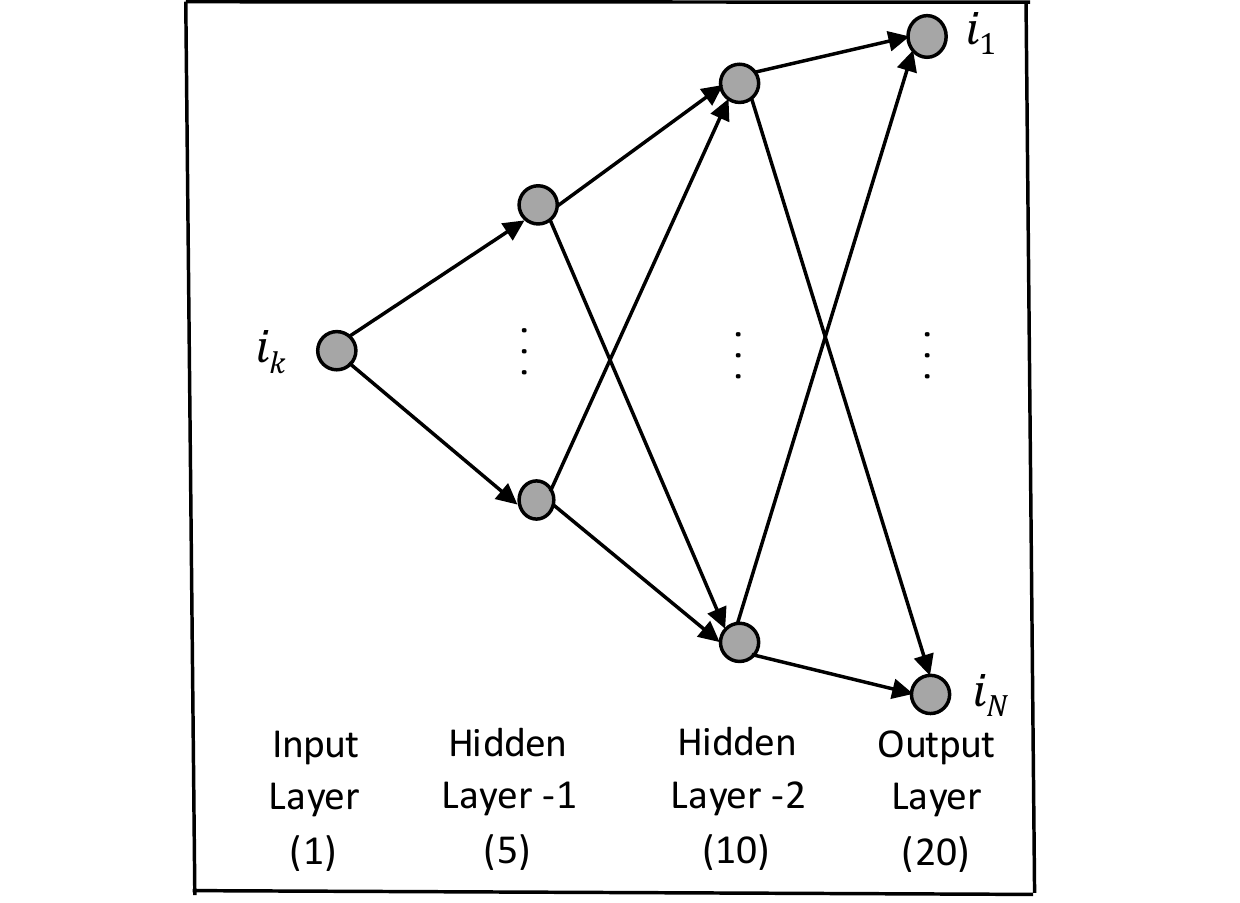}}
\caption{Learning the UE locations in cluster-basis, which are formed in polar coordinates, via training a DNN that can exploit the correlations among SINRs determined by the environment: (a) The clusters in polar coordinates for one sector of a macrocell. (b) Fully connected DNN architecture to learn all cluster values at the output from a single (input) cluster value, in which $N=20$.}
\end{figure}

A supervised fully connected DNN model is designed that is composed of an input layer, two hidden layers and an output layer as shown in Fig. \ref{fig:AE_arch}. The input layer takes a single cluster value and produces all cluster values via the hidden layers. To train this DNN model, all cluster values are measured many times offline, which yields many cluster vectors. In what follows, one specific entry -- which can be determined arbitrarily -- of all cluster vectors is picked, and used as the input data. More precisely, each input data sample and the corresponding cluster vector constitute one training sample, in which the latter is the labeled data. The proposed DNN model is trained according to this training data set using gradient descent and the backpropagation algorithm. In the online phase, any small cell whose location is known by a macrocell, measures the SINR of their UEs and reports the average value to the macrocell, which corresponds to that cluster value. Then, all the cluster values are learned via the DNN. Our model hinders the frequent updates of the offline cluster values, since the pattern among clusters changes rarely, due to significant environmental changes. The further details of the model that shows the type and size of each layer as well as the activation functions are depicted in Table \ref{tab:AE}. Here, the rectified linear unit (ReLU) is employed in the hidden layers to capture the non-linear relations.

\begin{table} [!h] 
\renewcommand{\arraystretch}{1.3}
\caption{The DNN model to learn the cluster values}
\label{tab:AE}
\centering
\begin{tabular}{c|c|c|c}
    \hline
    Layer & Type & Size & Activation\\
    \hline
    \hline
    In & Inputs & $1$ & - \\
    \hline
    Hidden-1 & Fully Connected & $N/4$ & ReLU\\
    \hline
    Hidden-2 & Fully Connected & $N/2$ & ReLU\\
    \hline
   Out & Outputs & N & Linear \\	
    \hline
\end{tabular}
\end{table}

\subsection{Pseudo-code of the Proposed Algorithm}
The pseudo-code of the proposed algorithm is given in Algorithm \ref{algorithm3}. Accordingly the interference values are learned in the offline phase for each state-action pair. In the online phase, the parameters of the DNN model responsible for learning the UE locations are first trained. In what follows, the $\bold{w}$ is set to zero, and then iteratively updated based on observing state, taking actions and taking rewards as illustrated in Fig. \ref{fig:seqOrder}. At each training iteration actions are selected using the $\epsilon$-greedy policy with the proposed scheduling in (\ref{epsilon-scheduling}) and the features in (\ref{refVec2}). Once $\textbf{w}$ is trained, the optimum action can be found for all states irrespective of how large the state-action pairs are thanks to the function approximation. This apparently alleviates the-curse-of-dimensionality problem. In the operational phase, the quantized and averaged SINRs of the typical UEs are found. This gives the input state, and the optimum action is selected for this state. 
\begin{algorithm}
 \caption{The training algorithm for the proposed Q-learning algorithm}\label{algorithm3}
 \begin{algorithmic}[1]
\renewcommand{\algorithmicrequire}{\textbf{Input: }}
\renewcommand{\algorithmicensure}{\textbf{Output: }}
\REQUIRE $s \in \mathbb{S}$ 
\ENSURE  $a^* \in \mathbb{A}$
\\ \textit{Offline Phase}:  Multi-agent mean field RL
	\STATE Learn $\beta(s,a)$ as described in \eqref{MCinterference}, \eqref{targetMLaction} and \eqref{betadef}.
\\ \textit{Online Phase}: Single agent feature based Q-learning
	\STATE Train the DNN model in Table \ref{tab:AE}. 
	\STATE Set $\mu$, $\alpha$, $\gamma_{min}$, $\lambda_{u,v}$, $k$, $T_\epsilon$.
	\STATE Initialize the optimization parameters as $\theta_t=15^{\circ}$, $\theta_{3dB}=10^{\circ}$, $\phi_{3dB} =70^{\circ}$.
  \STATE Initialize $\textbf{w}$ to 0.
	\STATE Initialize the state according to the current quantized SINRs, i.e., $s=[\gamma_{1,1}^{(q)}[0] \cdots \gamma_{U,V}^{(q)}[0]]$. 
	\STATE Choose the action $a$ for $s$ using the step 1 and $\epsilon$-greedy policy in (\ref{epsilon-greedy}). 
  \FOR {$n =0:\text{training period}$}
  \STATE Take the selected action $a$. 
	\STATE Determine $x_i, \forall i$ in (\ref{refVec2}).
  \STATE Observe the next state $s'=[\gamma_{1,1}^{(q)}[n+1] \cdots \gamma_{U,V}^{(q)}[n+1]]$.
	\STATE Observe the immediate reward $r[n]$ in (\ref{rewardImm}).
	\STATE Update $\epsilon$ according to (\ref{epsilon-scheduling}).
  \STATE Choose the next action $a'$ using step 1 and $\epsilon$-greedy policy in (\ref{epsilon-greedy}). 
  \STATE Update $\textbf{w}$ according to (\ref{updatew2}).
	\STATE Normalize $\textbf{w}$ as $w_i \leftarrow w_i/\sum_iw_i$. 
  \STATE Set $s\leftarrow s'$.
  \STATE Set $a\leftarrow a'$.
  \ENDFOR
 \end{algorithmic} 
 \end{algorithm}
 
\subsection{Convergence and Complexity Analysis}
To prove the convergence of the proposed algorithm, we first show that both the offline multi-agent mean field RL algorithm and the online single agent feature based Q-learning with linear function approximation converge. For the former, \cite{ML-MARL} proves that treating all the neighboring agents of a target agent as a single virtual agent, and then solving this two agent RL problem with game theory converge to a Nash equilibrium. The convergence of the latter in the sense of Bellman optimality is well-known if there is a single agent \cite{Barto}. However, in the online phase there are multiple agents, and each of them picks an action according to its local observations without caring what other agents do. Hence, we need to show that the equilibrium attained in the offline phase with mean field RL is not disturbed despite the online individual behavior of macrocells.
\begin{proposition}
The equilibrium $\bar{a}_k$ in \eqref{MCinterference} attained for the $k^{th}$ macrocell in offline multi-agent mean field RL holds even if single agent RL is used in the online phase for the uniformly distributed small cells.
\end{proposition}
\begin{proof}
Suppose that the interference due to the small cells in the coverage region of a macrocell reduces the number of possible actions, i.e., each agent picks an action from a subset of the action space that are subject to less interference as
\begin{equation}
\mathbb{X}_j = \mathbb{A} \setminus \mathbb{B}_j
\end{equation}
where $\mathbb{B}_j$ is the set that leads to lower reward due to the interference. The distance between the online action picked in $\mathbb{X}_j$  and the offline optimum action $a_j$ is denoted by $\delta_j$ such that $ \delta_j = 0$ if $a_j \notin \mathbb{B}_j$. Since the small cells are distributed uniformly, 
\begin{equation}
\bigcup_{j=1}^{N_j}  \mathbb{B}_j = \mathbb{A}. 
\end{equation}
This refers to $\sum_{j=1}^{N_j} \delta_j \in \emptyset$, or
\begin{equation}
\bar{a}_k = \frac{1}{N_j}\sum_{j=1}^{N_j}a_j+\delta_j.
\end{equation}
Hence, the equilibrium $\bar{a}_k$ is preserved, which completes the proof.
\end{proof}

The complexity of the proposed algorithm is evaluated separately as the computational and sample complexity (or number of trials). For the computational complexity, finding the Nash equilibrium, which has exponential worst-case behavior for two-player games \cite{NashQ}, dominates the overall complexity. However, this is done in the offline phase. Thus, its complexity can be tolerated. For sample complexity, the offline step requires trials on the order of $|\mathbb{S}|^2|\mathbb{A}|^2$ samples, where the square term in the action space stems from having one target and one virtual agent. This is again in the offline phase, and hence this sample complexity is not a big issue. For the online step, there is not any known bound for the sample complexity and it depends on how good the features are for the relevant environment. Our empirical results show that hundreds of online trials are sufficient to have satisfactory performance for our definition of features, given in \eqref{refVec2}. 

\section{Simulations and Training} \label{Simulations}
The proposed deep RL algorithm based on Q-learning is evaluated with extensive simulations. In particular, its performance is compared with the optimum solution. This optimum solution is acquired by solving (\ref{optimizationProblemForm}) with a genie-aided method, which knows all the antenna gains of UEs for each antenna setting and the responses of neighboring cells, or solving \eqref{optimizationProblemRew} with the online multi-agent mean field RL algorithm. To have a fair comparison with the proposed algorithm, we allow to choose one antenna configuration for a large period of $N$ time intervals instead of taking the best action for each time interval for the optimum solution. Furthermore, the proposed deep RL algorithm is explicitly compared with the classical single agent RL, in which the target cell is totally nonadaptive to the other cells and treat them as a part of the environment. We have developed an LTE simulator using Python libraries and generated a multi-cell HetNet simulation environment. For this simulation environment the performance of the algorithm is assessed  through SciKit and TensorFlow libraries. In this section, the simulation environment is first clarified, then the states and actions for this specific simulation environment are explained, and the simulation results are provided at last.  

\subsection{Simulation Environment}
A two-tier HetNet is considered such that there are macrocells and picocells. All the nodes including the macrocells, picocells and UEs are randomly distributed over a square planar area, which is taken as $5\times5$ km$^2$.  Specifically, macrocells and picocells are distributed according to the Poisson Point Process with a density of $\lambda_m=0.25/$km$^2$ and $\lambda_p=2$/km$^2$, respectively. There are $400$ UEs that are distributed uniformly within the area of interest. The overall simulation environment is illustrated in Fig. \ref{fig:SimEnv}.
\begin{figure} [!h] 
\centering
\includegraphics [width=4.5in]{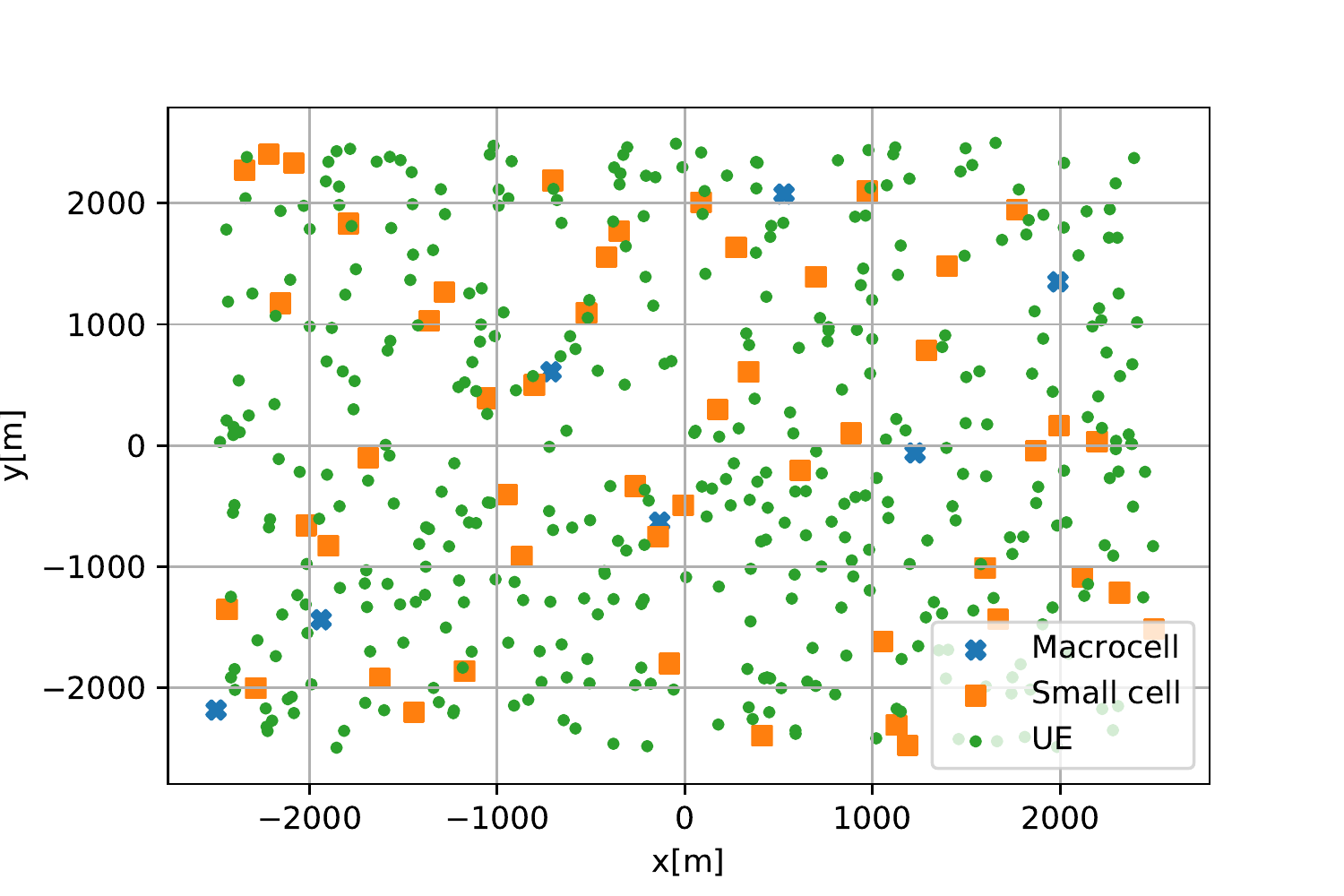}
\caption{A two-tier HetNet environment with macrocells and picocells, in which the uniformly distributed UEs are connected to the strongest cell.}\label{fig:SimEnv}
\end{figure}

UEs are assumed to be equipped with a single antenna, and connected to one of the cells according to the maximum received signal strength, which is mainly determined by  the transmission power, antenna gain, path loss and shadow fading. All these parameters for macrocells and picocells are set according to \cite{Macro-Pico}. Specifically, the transmission powers of macrocells are $46$ dBm with a maximum antenna gain of $15$ dBi. On the other hand, the picocells transmit at a maximum power of $24$ dBm using a omni-directional antenna gain with $0$ dBi. The path loss model for macrocells is assumed to be $128.1+37.6\log(d_m)$ where $d_m$ is in km scale, whereas it is $38 + 30\log(d_p)$ for picocells and $d_p$ is in meters. There is log-normal shadow fading in the environment, whose standard deviation is $10$ dB for macrocells and $6$ dB for picocells. 

Without loss of generality, the macrocell of the bottom left base station in Fig. \ref{fig:SimEnv} that covers the area between $0\degree$ and $120\degree$ is picked up for the performance measurements of the proposed deep RL algorithm. There are $5$ UEs connected to this macrocell, which are spread almost the entire coverage region. That is, the UEs are not clustered in a specific small sub-area under the coverage region. This is a quite complicated scenario to dynamically tune the antenna parameters according to the UEs and the time-varying environment including the effects of other cells. Hence, it can be interesting to observe the performance of the proposed deep RL algorithm for this scenario.

\subsection{Training}
To train the proposed deep RL algorithm, the states and actions of this macrocell have to be clarified. Specifically, the states are defined by uniformly quantizing the SINRs of UEs with $2$ dB resolution, in which the minimum SINR is $0$ dB and the maximum SINR is $12$ dB. This results in $7^5=16,807$ states for a selected macrocell that has (for example) $5$ UEs, wherein $s_0=[0, 0, 0, 0, 0]$, $s_1=[0, 0, 0, 0, 2]$, and $s_{16,806}=[12, 12, 12, 12, 12]$. The action space, which is composed of the possible antenna configurations, is defined in terms of $\theta_t$, $\theta_{3dB}$, $\phi_{3dB}$ in Table \ref{tab:actions}. Accordingly, $|\theta_{t}|=6$, $|\theta_{3dB}|=5$, $|\phi_{3dB}|=6$ leading to  $180$ different actions. The actions are ordered as $a_0={(0\degree,  4.4\degree, 45\degree)}, a_1=(0\degree,  4.4\degree, 55\degree), \cdots, a_6=(0\degree,  6.8\degree, 45\degree), \cdots, a_{179}=(15\degree,  13.5\degree, 85\degree)$. The algorithm is trained with these definitions of states and actions according to the aforementioned simulation environment for a certain period of time.
\begin{table*}[!h] 
\renewcommand{\arraystretch}{1.3}
\caption{The possible antenna configurations in terms of $\theta_t$, $\theta_{3dB}$, $\phi_{3dB}$ that constitute the actions}
\label{tab:actions}
\centering
\begin{tabular}{c|c|c}
    \hline
    Parameter & Notation &  Possible Values\\
    \hline
    \hline
     Tilt angle & $\theta_t$   &   $\theta_t=\{0\degree, 3\degree, 6\degree, 9\degree, 12\degree, 15\degree\}$\\ 
    \hline
    Vertical $3$dB beamwidth & $\theta_{3dB}$   &   $\theta_{3dB}=\{4.4\degree, 6.8\degree, 9.4\degree, 10\degree, 13.5\degree\}$\\
    \hline
    Horizontal $3$dB beamwidth & $\phi_{3dB}$   &  $\phi_{3dB}=\{45\degree, 55\degree , 65\degree, 70\degree, 75\degree, 85\degree\}$\\
    \hline
\end{tabular}
\end{table*}

After training, the parameters of the algorithm are set, which enables the agent to take an optimum action. The optimum action for a given state lasts for a long period of time until there is a sustained environmental change that can greatly degrade the performance. If this happens, the algorithm has to be retrained to set the parameters according to the new environment. Note that this retraining does not have to be done from scratch. To illustrate, transfer learning can be employed to obtain the new parameters with minimum number of training samples \cite{Pan}. 

\subsection{Performance}
The proposed algorithm performance is assessed for the aforementioned simulation environment with sparse network knowledge just by maximizing the long-term expected rewards according to the designed state and action spaces with moderate number of online training samples. Specifically, $200$ online training samples are used to set the parameters of $\textbf{w}$ after the offline training. This is quite reasonable considering the fact that there has to be at least $\kappa*16,807*180$ online samples for conventional Q-learning (if there is not any offline training), where $\kappa$ is some polynomial bound \cite{Kearns2002}. Without loss of generality, the hyper-parameters are selected as $\mu=0.8, \alpha = 0.9, \gamma_{min} = 2, \lambda_{u,v}=1, \forall_{u,v}$. Furthermore, $T_{\epsilon}$ is taken $10$ for the designed scheduling in (\ref{epsilon-scheduling}), and $\textbf{w}$ is initialized to all zeros before online training.

The performance of the proposed algorithm is compared with the optimum solution in an attempt to see what percentage of the optimum solution, which is quantified in terms of a SINR gain, can be obtained. For this purpose, we normalize the performance so that $1$ refers to perfectly having the optimum performance. The plot that shows the performance in terms of relative variance of the environment dynamics, which is defined in \eqref{env_dyn}, is depicted in Fig. \ref{fig:PerfTrade}. There are $2$ important takeaways that can be inferred from this plot. First, for low relative variance the proposed algorithm provides near-optimum performance. This makes total sense. The dynamics of the environment in the offline phase becomes exactly the same as for the online environment when the variance is $0$, which also makes the relative variance $0$, and hence we can achieve the optimum performance. When the variance increases, the mismatch between the offline and online environments increases and the performance decreases due to using a small number of training samples. Second, there is apparently some performance guarantee for the proposed algorithm, because the performance is saturated around $0.6$ of the optimum even at very high relative variance. That is, even if there is a performance degradation for increasing variance, the performance does not drop below $60\%$ of the optimum solution. This can be simply quantified as having $0.6\Delta$ dB SINR gain with hundreds of trials instead of getting $\Delta$ dB SINR gain with millions of trials.
\begin{figure} [!t] 
\centering
\includegraphics [width=4.75in]{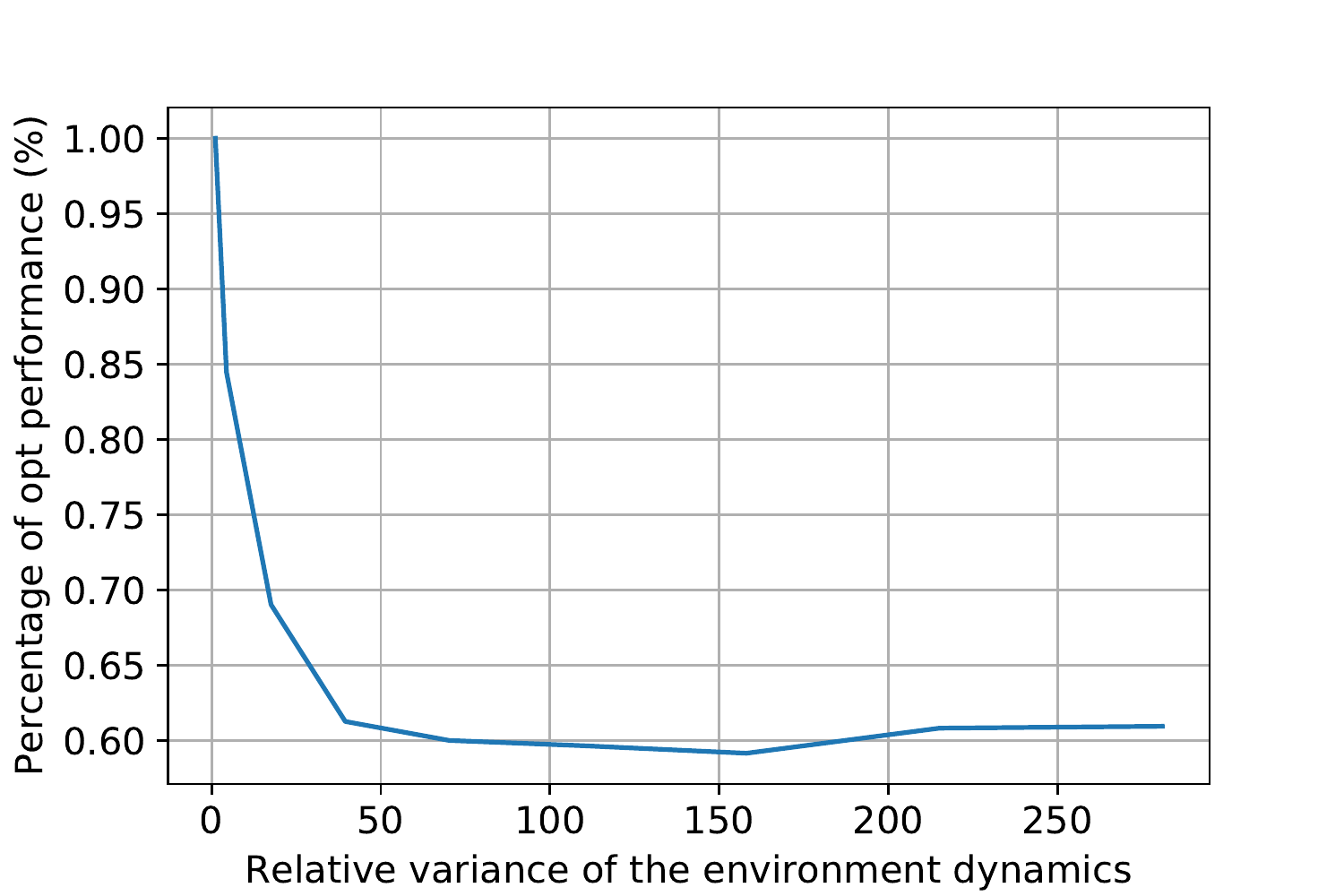}
\caption{The performance of the proposed algorithm with respect to the optimum solution depending on the relative variance of the environment dynamics.}\label{fig:PerfTrade}
\end{figure}

To observe the efficiency of the proposed deep RL algorithm, which is a compromise between multi-agent and single agent RL, its performance is compared with the following two cases in terms of SINR gain when: (i) an online multi-agent mean-field RL algorithm is used, which also refers to the optimum solution, (ii) a classical single agent mean-field RL algorithm is used without any offline phase. This result is presented in Fig. \ref{fig:SINR_gain}, which clearly indicates that our algorithm offers a good trade-off between a complex multi-agent mean field RL that requires millions of online trials and a global network knowledge, and a highly suboptimum single-agent RL.
\begin{figure} [!t] 
\centering
\includegraphics [width=4.75in]{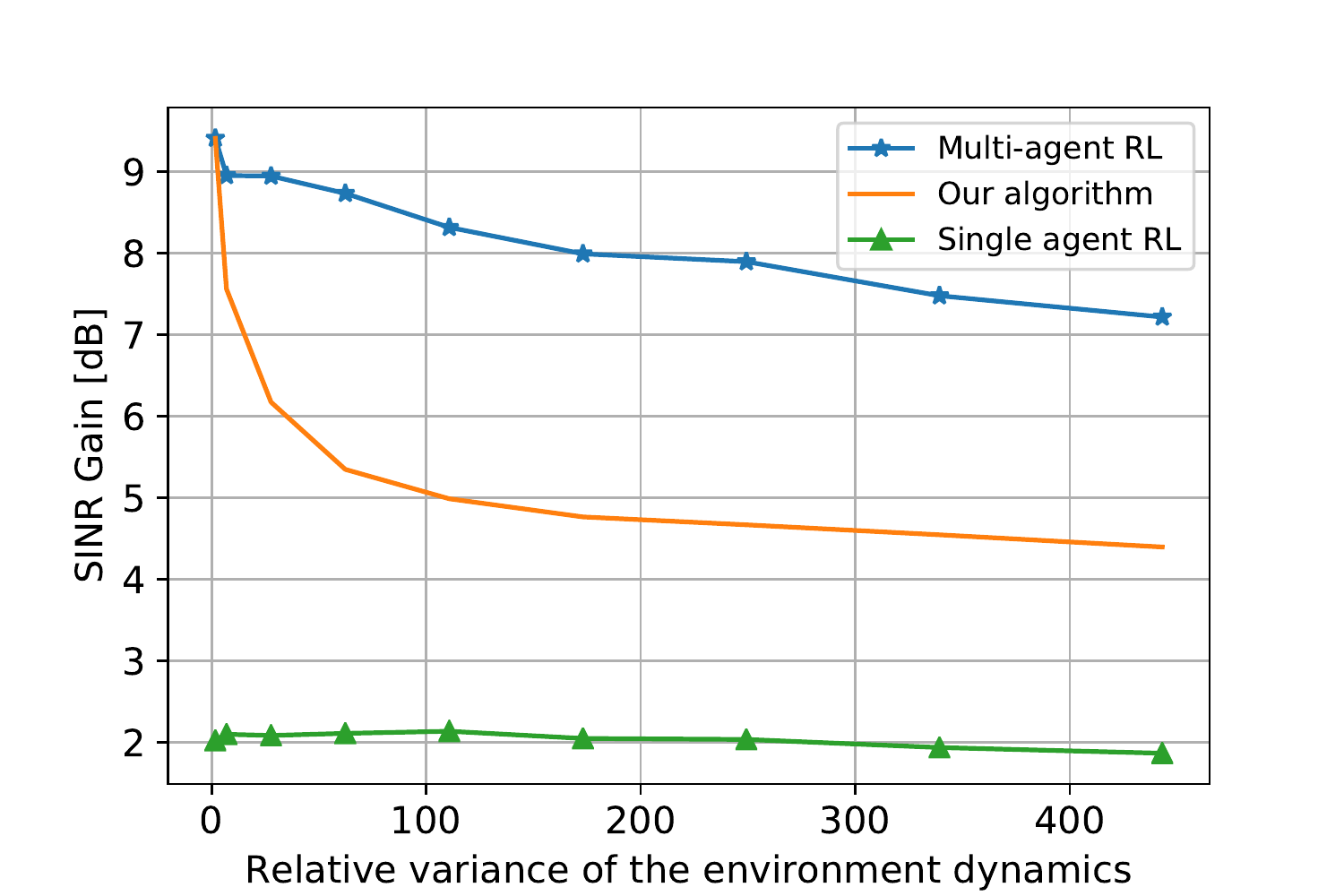}
\caption{Comparison of the proposed algorithm with online multi-agent mean-field RL and single agent RL in terms of the SINR gain.}\label{fig:SINR_gain}
\end{figure}

A key metric besides the performance in RL-based algorithms for general wireless networks is the compactness of the algorithm. That is, how different states map to actions is important, especially for the case of antenna tuning, in which changing the antenna parameters after training is too costly. Here, what is desired is to map all the states into a small subset of actions to reduce the need for the antenna configuration change when the state of the system alters. The compactness of our algorithm is illustrated in Fig. \ref{fig:CompactPol}. As can be seen, nearly $80\%$ of the states are mapped to actions whose indices are between $70$ and $90$. This compactness is mainly associated with using a function approximation and low mobility. To be more precise, the former interpolates the close states to the same action, and the latter leads to have the almost same geometry for all states, and macrocells tend to select the action according to the geometry. That is, if there is not any interference and users are static, all the states are mapped to a single action. For slowly varying environments that have low mobility, mapping the same action to different states intelligently (i.e., by learning a function approximation) does not lead to significant performance loss. On the other hand, for the environments that have high variance with highly mobile UEs, a tradeoff between the tuning sensitivity and the performance optimality can surface. Notice that our algorithm needs the UE locations while using the features. Next, we evaluate how efficiently the macrocells can learn the UE locations.
\begin{figure} [!h] 
\centering
\includegraphics [width=4.75in]{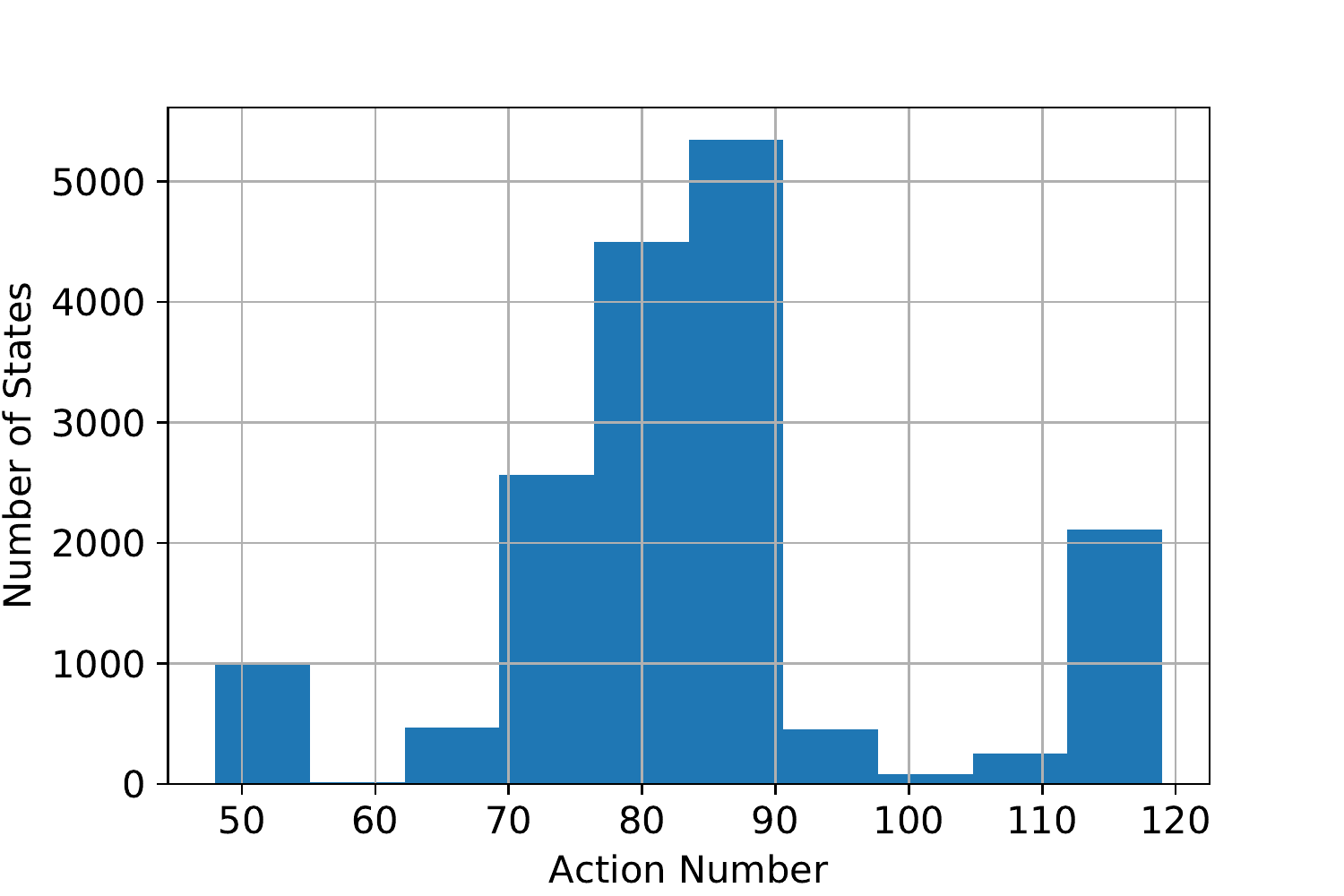}
\caption{Mapping the $16,807$ states to $180$ actions according to the proposed deep RL algorithm based on Q-learning.}\label{fig:CompactPol}
\end{figure}

\subsection{Learning the UE Locations}
The locations of UEs are learned by placing them to the correct clusters as presented in Fig. \ref{fig:PolarAE}. In this direction, the coverage region of the macrocell is divided into $20$ virtual clusters with $0.1$ km resolution in radius and $30\degree$ resolution in angle. This cluster definition only leads to $0.68$ dB loss in average antenna gain according to (\ref{antenna_gain}) when UEs are located at the edge of the clusters and we consider them at the center. This seems reasonable considering that rewards are received in $2$ dB resolution. 

To find the locations of UEs, we need to learn the online values of the clusters via the proposed supervised DNN model in Table \ref{tab:AE}. In this experiment, a synthetic dataset is generated to train and test the proposed DNN model using the simulation environment in Fig. \ref{fig:SimEnv}, in which the data used for training and test is separated. Once DNN is trained, UEs are assigned to one of these clusters according to their averaged SINRs. The main reason for using the average SINR instead of the instantaneous SINR lies in the fact that average SINR can give a closer value with respect to the corresponding cluster value. Notice that this is only possible for low mobility, which is the case in our optimization problem. Furthermore, the impact of small-scale fading disappears due to this averaging, and hence small-scale fading does not affect the learning accuracy.  

The accuracy of the proposed model is found by comparing the average SINRs of UEs with the ground truth cluster values in the test data according to the minimum distance criterion. Specifically, if the average SINRs of UEs is the closest with the correct cluster in the ground truth, this means success; otherwise an error comes. To better understand the benefit of the proposed DNN model, its performance is compared with the conventional fingerprinting approach, in which the offline cluster values are averaged and then UEs are assigned to one of the clusters according to the minimum distance criterion directly, i.e., without exploiting a pattern. This comparison is provided in Fig. \ref{fig:Accuracy} indicating that the proposed model can give a very accurate result as long as a sufficient number of training samples are employed. 
\begin{figure} [!h] 
\centering
\includegraphics [width=4.75in]{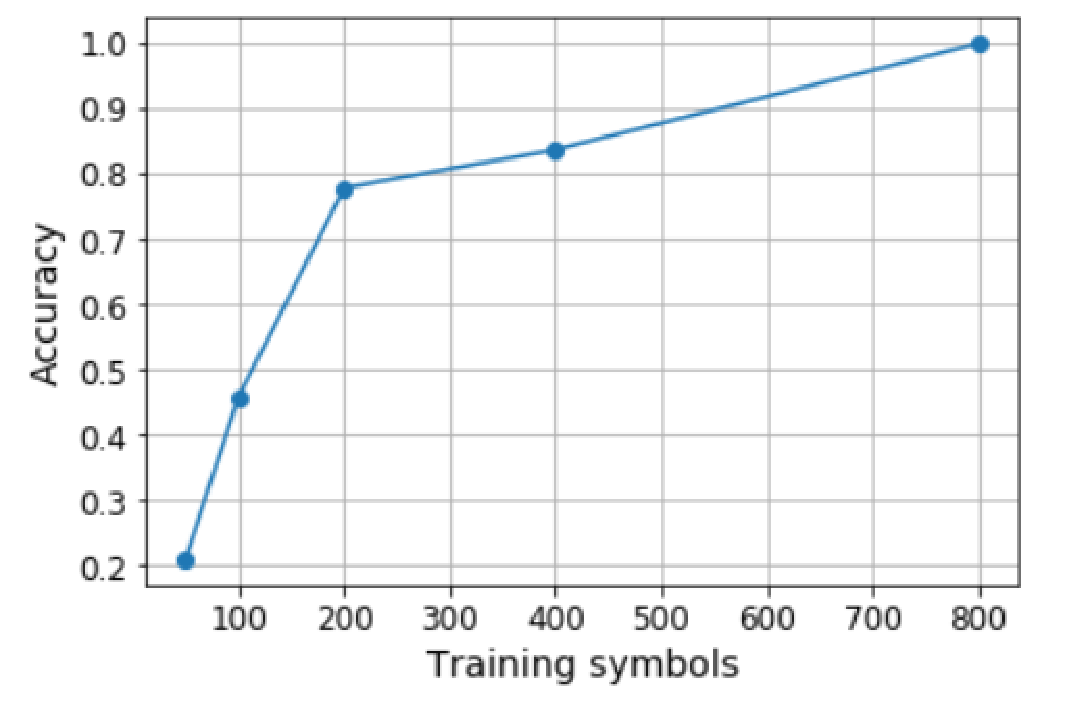}
\caption{The accuracy of finding the locations of UEs by learning the pattern among cluster values.}\label{fig:Accuracy}
\end{figure}

\section{Conclusions}\label{Conclusions}
In this paper, we aim to dynamically optimize the antenna parameters of macrocells in HetNets in order to maximize the weighted sum-rate of the users in the macrocell. This optimization problem is first turned into an RL problem via defining states, actions and rewards according to the available control signals in macrocells. Then, a practical algorithm is developed to tackle the large sample complexity and curse-of-dimensionality problems. Specifically, we use DL and a mean field multi-agent RL approach to design a novel and practical deep RL algorithm that needs only hundreds of samples in an online environment. Despite such a small number of online samples, the proposed algorithm can give performance approaching the optimum distributed solution for low relative variance of the environmental dynamics. Even in the case of high relative variance, the proposed algorithm appears to ensure some performance gain regardless of how high the variance goes. As future work, determining the performance gap between the distributed optimum solution and the centralized global optimum solution and shrinking this gap by some messaging among cells would be valuable. Additionally, it could be an interesting future work to utilize the proposed RL framework to dynamically learn the coefficients of phase shifters associated with the number of antennas. The proposed RL framework and algorithm can also be applied for other parameter optimizations such as transmit power.

\appendix[Basics of Q-learning] \label{Q-learning}
The basic idea in Q-learning is to update the quality or q-values of state-action pairs iteratively according to the rewards by following a policy $\pi_n$. More precisely,
\begin{equation} \label{qFunc}
q_{\pi_n}(s,a)\leftarrow q_{\pi_n}(s,a) + \mu(r[n]+\alpha q_{\pi_n}(s',a') - q_{\pi_n}(s,a))
\end{equation}
where $\mu$ is the learning rate, $\alpha$ is the discount factor for future rewards, $s'$ and $a'$ are the next state and action, and $s$ and $a$ are the current state and action. The optimum policy $\pi^*$  for (\ref{qFunc}) can be found after $T$ iterations as
\begin{equation} 
\lim_{n\rightarrow T}{\pi_n}=\pi^*
\end{equation}
through the $\epsilon$-greedy algorithm where $\epsilon$ provides a balance between exploration and exploitation \cite{Barto}. In the $\epsilon$-greedy algorithm, the maximum q-valued action is chosen for a given state with probability $1-\epsilon$, or a random action is taken with probability $\epsilon/|\mathbb{A}|$ at each step so that
\begin{equation} \label{epsilon-greedy}  
  \pi_n(s,a)=
  \begin{cases}
    \epsilon/|\mathbb{A}| + 1 - \epsilon, & \text{if $a^* = \arg \max_a\ q_{\pi_n}(s,a)$} \\
     \epsilon/|\mathbb{A}|, & \text{o.w.} \\
  \end{cases}
\end{equation}
where $\pi_n(s,a)$ shows the probability of selecting the action $a$ for the state $s$ under the policy $\pi_n$.

The key challenge to find the optimum policy for (\ref{optimizationProblemRew}) with Q-learning is the large number of state-action pairs. This not only increases the convergence time but also wastes resources due to trial and error mechanism based on feedback. To cope with the large number of state-action pairs, approximating the q-values with a function seems appropriate. The main advantage of this method is that there is no need to see all state-action pairs, which provides faster convergence. The q-values can be approximated in terms of a vector $\textbf{w}$ such that 
\begin{equation} \label{paramQ}
q_{\pi_n}(s,a) \approx \hat{q}(s,a,\textbf{w})
\end{equation}
and once $\textbf{w}$ is trained, q-values can be directly determined for every state and action. Additionally, any function approximator for Q-learning generalizes the learning from seen state-action pairs to unseen pairs, which is based on the intuition that ``nearby'' states should behave similarly, i.e., select the same action.

\end{document}